\documentclass[10pt]{article}

\usepackage[preprint]{tmlr}

\usepackage{amsmath}
\usepackage{amssymb}
\usepackage{amsthm}
\usepackage{mathtools}
\usepackage{booktabs}
\usepackage{microtype}
\usepackage{xcolor}
\usepackage{graphicx}
\usepackage{url}
\usepackage[ruled,vlined,algo2e]{algorithm2e}
\usepackage[pagebackref=false,breaklinks=true,colorlinks=true,linkcolor=blue,citecolor=blue,urlcolor=blue]{hyperref}

\theoremstyle{plain}
\newtheorem{theorem}{Theorem}
\newtheorem{lemma}{Lemma}
\newtheorem{proposition}{Proposition}

\newtheorem{assumption}{Assumption}

\theoremstyle{definition}
\newtheorem{definition}{Definition}

\theoremstyle{remark}
\newtheorem{remark}{Remark}

\title{Dimension-Free PAC-Bayes Generalization for Deep Neural Networks via  Dirichlet Processes}

\author{\name Sumit Vashishtha \email vash.sumit@gmail.com \\
      \addr Independent Researcher}

\begin{document}

\maketitle

\begin{abstract}
Classical PAC-Bayesian generalization theory places prior distributions over Euclidean parameter space $\Theta \subseteq \mathbb{R}^D$, causing relative entropy to scale with network dimension $D$. For modern overparameterized neural networks where $D \gg n$, this parameter-counting dependency yields vacuous generalization bounds ($> 1{,}000\%$). In this work, we resolve this barrier by shifting prior specification from internal parameters to the non-parametric space of data distributions $\mathcal{M}_1(\mathcal{Z})$ via a \textbf{Dirichlet Process (DP)}, $\mathcal{DP}(\alpha, F_0)$, inducing parameter distributions through push-forward measures under the learning algorithm. By the Data Processing Inequality, parameter-space relative entropy is strictly bounded by the data-space DP divergence, which we evaluate in exact closed form using Digamma functions. The resulting PAC-Bayesian certificate is strictly invariant to network parameter count $D$, width, and depth. On intrinsic $d^*$-dimensional submanifolds, optimal partition scaling achieves the minimax non-parametric rate $\mathcal{O}(n^{-1/(d^*+2)} \sqrt{\ln n})$. Across vision architectures scaling up to 4.2 million parameters ($D = 7{,}850$ to $D = 4{,}203{,}522$) on the full MNIST benchmark, while classical weight-space bounds scale past $940\%$ (becoming vacuous), our data-space DP-PAC-Bayes certificate remains stable, non-vacuous, and certified at \textbf{$8.8\%$--$9.4\%$}.
\end{abstract}

\section{Introduction}
\label{sec:introduction}

Quantifying epistemic uncertainty and explaining why heavily overparameterized deep networks generalize to unseen test distributions are two of the most critical open challenges in machine learning \citep{bartlett2017spectrally, neyshabur2018pac, dziugaite2017computing, fortuin2022priors}. In safety-critical applications, deep learning systems must not only produce accurate predictions on the training support, but also exhibit faithful epistemic uncertainty by signaling high uncertainty in data voids and out-of-distribution (OOD) regions. 

The standard Bayesian paradigm seeks to capture this uncertainty through \textbf{Bayesian Neural Networks (BNNs)} \citep{neal1995bayesian, blundell2015weight}, placing a prior distribution $\pi_0$ over the network parameters $\Theta \subseteq \mathbb{R}^D$ and inferring a posterior $\pi_n$. Concurrently, the \textbf{PAC-Bayesian framework} \citep{mcallester1999pac, catoni2007pac, maurer2004note} establishes finite-sample generalization certificates for randomized predictors, upper-bounding population risk via the empirical risk and the Kullback--Leibler divergence $\mathrm{KL}(\pi_n \,\|\, \pi_0)$:
\begin{equation}
\label{eq:pac-intro-bound}
R(\pi_n) \le \hat{R}_S(\pi_n) + \mathcal{O}\left(\sqrt{\frac{\mathrm{KL}(\pi_n \,\|\, \pi_0) + \ln(1/\delta)}{n}}\right).
\end{equation}
Because Equation~\eqref{eq:pac-intro-bound} holds universally for any data-dependent posterior $\pi_n$, it has long been viewed as the premier theoretical tool for certifying deep neural network generalization \citep{dziugaite2017computing, alquier2024user}.

\paragraph{Structural Limitations of Weight-Space BNNs.}
Despite extensive research, classical weight-space BNNs face two fundamental structural limitations:
\begin{enumerate}
    \item \textbf{Geometric and Functional Misalignment:} Placing isotropic Gaussian priors $\mathcal{N}(0, \sigma_0^2 I_D)$ on network weights is disconnected from the functional geometry of the task \citep{fortuin2022priors}. Neural networks exhibit extensive permutation symmetries ($D!$ distinct parameter vectors compute the identical input-output function), and weight perturbations do not translate into semantically meaningful changes in functional space. Crucially, weight-space priors lack spatial awareness: they cannot distinguish between regions densely supported by training data and unobserved cavity voids, frequently causing epistemic uncertainty to collapse or become miscalibrated.
    \item \textbf{Parameter-Dependent Vacuity in Deep Architectures:} When deep networks compose non-linear activations across $L$ layers, parameter perturbations are amplified multiplicatively throughout the computational graph \citep{bartlett2017spectrally}. To ensure that the randomized predictor preserves acceptable training accuracy, the tolerable posterior perturbation variance must shrink exponentially with depth, $\sigma_{\mathrm{post}} \le \mathcal{O}(W^{-(L-1)})$ for layer width $W$. In parameter space, this causes the relative entropy to scale as:
    \begin{equation}
    \label{eq:intro-weight-kl}
    \mathrm{KL}(\pi_n \,\|\, \pi_0) = \Omega\big(D \cdot L \ln W\big).
    \end{equation}
    For modern architectures with $D \approx 10^5$--$10^9$ parameters, Equation~\eqref{eq:intro-weight-kl} yields relative entropies exceeding millions of nats. When substituted into Equation~\eqref{eq:pac-intro-bound}, the resulting generalization certificates exceed $100\%$ (often $> 1,000\%$), rendering classical weight-space PAC-Bayes vacuous in overparameterized regimes.
\end{enumerate}

\paragraph{Core Thesis: Formulating Epistemic Priors in Data Space.}
We argue that this dimensional scaling barrier is not an inherent limitation of PAC-Bayesian theory, but a consequence of placing prior distributions over the internal parameter space $\mathbb{R}^D$. Practitioners possess meaningful epistemic beliefs not about latent weight vectors, but about the \textbf{data-generating distribution} on the instance domain $\mathcal{Z} = \mathcal{X} \times \mathcal{Y}$ (e.g., domain support, output bounds, and continuity).

Building upon the non-parametric data-space formulation of Dirichlet Process (DP) priors \citep{vashishtha2026dirichlet, vashishtha2025leveraging, vashishtha2026bayesian}, we propose a data-space Bayesian neural network (\textbf{DP-BNN}) framework. Rather than perturbing weights, we specify a non-parametric prior directly on the space of Borel probability measures $\mathcal{M}_1(\mathcal{Z})$ via a \textbf{Dirichlet Process (DP)} $\mathcal{P}_0 = \mathcal{DP}(\alpha, F_0)$ \citep{ferguson1973bayesian}. Conditioned on an observed training sample $S = (z_1, \dots, z_n)$, the conjugate posterior is $\mathcal{Q}_n = \mathcal{DP}(\alpha + n, \frac{\alpha F_0 + n \mathbb{P}_n}{\alpha + n})$. The network parameter distribution is the \textbf{push-forward measure} $\pi_n = \mathcal{A}_\# \mathcal{Q}_n$ induced by a deterministic learning algorithm $\theta = \mathcal{A}(P)$.

By formulating prior distributions in data space, this parameter-counting dependence is eliminated. By the \textbf{Data Processing Inequality}, the parameter-space relative entropy is strictly bounded by the data-space DP divergence:
\begin{equation}
\mathrm{KL}(\pi_n \,\|\, \pi_0) \le \mathrm{KL}(\mathcal{Q}_n \,\|\, \mathcal{P}_0) < \infty.
\end{equation}
Because $\mathcal{Q}_n$ and $\mathcal{P}_0$ reside in data space $\mathcal{M}_1(\mathcal{Z})$, their divergence depends exclusively on the sample size $n$, concentration $\alpha$, and partition statistics of the instance domain---it is \textbf{strictly invariant to the neural network parameter dimension $D$}, depth $L$, and width $W$.

\vspace{0.8em}
\noindent\fbox{
\parbox{0.96\textwidth}{
\textbf{Informal Statement of the Main Result (Theorem~\ref{thm:main-certificate} \& Lemma~\ref{lem:optimal-partition}).} 
\emph{Let $\mathcal{P}_0 = \mathcal{DP}(\alpha, F_0)$ be a DP prior on $\mathcal{Z}$, and let $\mathcal{Q}_n$ be the conjugate posterior after observing $n$ samples. Let $\pi_0 = \mathcal{A}_\# \mathcal{P}_0$ and $\pi_n = \mathcal{A}_\# \mathcal{Q}_n$ be the parameter distributions induced by any Borel-measurable learning algorithm $\mathcal{A}: \mathcal{M}_1(\mathcal{Z}) \to \Theta$. Then, with probability at least $1 - \delta$ over $S \sim \mathcal{D}^n$:}
\begin{equation*}
R(\pi_n) \le \hat{R}_S(\pi_n) + \sqrt{\frac{\mathrm{KL}_{\mathcal{P}}(\mathcal{Q}_n \,\|\, \mathcal{P}_0) + \ln \frac{2\sqrt{n}}{\delta}}{2n}},
\end{equation*}
\emph{where $\mathrm{KL}_{\mathcal{P}}(\mathcal{Q}_n \,\|\, \mathcal{P}_0)$ is computed in exact analytical closed form over any measurable partition $\mathcal{P}$. Under an optimal covering partition balancing metric entropy discretization against Dirichlet estimation variance, the excess risk achieves the minimax optimal non-parametric rate $\mathcal{O}(n^{-\frac{1}{d^*+2}} \sqrt{\ln n})$ on $d^*$-dimensional submanifolds, completely independent of the network parameter count $D$.}
}}
\vspace{0.8em}

\paragraph{Contributions.}
\begin{enumerate}
    \item \textbf{Data-Space DP-BNN Epistemic Formulation:} We formulate Bayesian neural networks via DP priors on data space $\mathcal{M}_1(\mathcal{Z})$, establishing that data-space priors eliminate representation redundancy while capturing calibrated epistemic uncertainty in data voids.
    \item \textbf{Weak-Topology Borel Measurability Theorem:} We prove that standard deep learning mappings---including strictly convex regularized empirical risk minimization (ERM) and finite-step gradient descent flow---are continuous with respect to the weak topology on $\mathcal{M}_1(\mathcal{Z})$ via Berge's Maximum Theorem, guaranteeing rigorous Borel measurability (Proposition~\ref{prop:measurability}).
    \item \textbf{Information-Theoretic Contraction via DPI:} By the monotonicity of Csisz\'ar $f$-divergences under sub-$\sigma$-algebras, we prove that parameter-space relative entropy is strictly upper-bounded by data-space DP divergence (Theorem~\ref{thm:dpi-invariance}), eliminating parameter dimension $D$.
    \item \textbf{Exact Closed-Form Partition Divergence:} We derive an exact, analytical closed-form expression for $\mathrm{KL}_{\mathcal{P}}(\mathcal{Q}_n \,\|\, \mathcal{P}_0)$ over arbitrary measurable partitions using Gamma and Digamma functions (Theorem~\ref{thm:closed-form-kl}).
    \item \textbf{Optimal Partition Selection Lemma:} We formulate and solve the non-parametric bias-variance trade-off between geometric discretization error and Dirichlet cell estimation variance, deriving the optimal covering scale $\epsilon^*$ and partition size $K^*$ achieving the minimax rate $\mathcal{O}(n^{-1/(d^*+2)} \sqrt{\ln n})$ on $d^*$-dimensional submanifolds (Lemma~\ref{lem:optimal-partition}).
    \item \textbf{Resolution of the Overparameterization Paradox:} We provide a formal resolution for why overparameterized deep networks generalize: expanding parameter capacity $D \to \infty$ drives empirical training loss $\hat{R}_S \to 0$ without increasing the data-space divergence penalty.
    \item \textbf{Empirical Verification on Full Vision Benchmarks:} Across six vision architectures scaling up to $D = 4{,}203{,}522$ parameters on MNIST, while classical weight-space PAC-Bayes bounds exceed $940\%$ (becoming vacuous), our data-space DP-BNN certificate remains non-vacuous and certified at \textbf{$8.8\%$--$9.4\%$} across all evaluated architectures.
\end{enumerate}

\section{Formal Setting and the Weight-Space Dilemma}
\label{sec:formal-setting}

\subsection{Statistical Learning Setup}
Let $(\mathcal{Z}, \mathcal{B})$ be a standard Borel measurable space, where $\mathcal{Z} = \mathcal{X} \times \mathcal{Y}$, $\mathcal{X} \subseteq \mathbb{R}^d$ is the feature space, and $\mathcal{Y}$ is the target space. Let $\mathcal{M}_1(\mathcal{Z})$ denote the set of all Borel probability measures on $\mathcal{Z}$. We assume the existence of an unknown data-generating distribution $\mathcal{D} \in \mathcal{M}_1(\mathcal{Z})$. The learner observes an independent and identically distributed (i.i.d.) training sample:
\begin{equation}
S = (z_1, \dots, z_n) = \big((x_1, y_1), \dots, (x_n, y_n)\big) \sim \mathcal{D}^n.
\end{equation}
The empirical measure associated with $S$ is $\mathbb{P}_n = \frac{1}{n} \sum_{i=1}^n \delta_{z_i}$.

Let $\Theta \subseteq \mathbb{R}^D$ denote the parameter space, and let $\mathcal{F} = \{f_\theta : \theta \in \Theta\}$ be a class of hypotheses where $f_\theta: \mathcal{X} \to \mathcal{Y}$. We consider a loss function $\ell: \mathcal{Y} \times \mathcal{Y} \to [0, 1]$.

\begin{assumption}[Bounded Loss]
\label{ass:bounded-loss}
The loss function $\ell$ satisfies $\ell(y', y) \in [0, 1]$ for all $y', y \in \mathcal{Y}$.
\end{assumption}

For any parameter vector $\theta \in \Theta$, the true population risk and empirical risk are defined respectively as:
\begin{equation}
R(\theta) \coloneqq \mathbb{E}_{z \sim \mathcal{D}}[\ell(f_\theta(x), y)], \qquad \hat{R}_S(\theta) \coloneqq \frac{1}{n} \sum_{i=1}^n \ell(f_\theta(x_i), y_i).
\end{equation}
For any probability distribution $\pi$ on $(\Theta, \mathcal{B}_\Theta)$, we denote the randomized risks as $R(\pi) = \mathbb{E}_{\theta \sim \pi}[R(\theta)]$ and $\hat{R}_S(\pi) = \mathbb{E}_{\theta \sim \pi}[\hat{R}_S(\theta)]$.

\subsection{Classical PAC-Bayes Bounds}
We recall the foundational PAC-Bayesian theorem under binary relative entropy:

\begin{theorem}[\citet{catoni2007pac, maurer2004note}]
\label{thm:catoni-maurer}
Let Assumption~\ref{ass:bounded-loss} hold. Let $\pi_0$ be any prior distribution on $(\Theta, \mathcal{B}_\Theta)$ chosen independently of the training sample $S$. Then for any $\delta \in (0, 1)$, with probability at least $1 - \delta$ over the random draw of $S \sim \mathcal{D}^n$, the following inequality holds simultaneously for all posterior distributions $\pi_n$ on $\Theta$:
\begin{equation}
\label{eq:catoni-bound}
\mathrm{kl}\big(\hat{R}_S(\pi_n) \,\|\, R(\pi_n)\big) \le \frac{\mathrm{KL}(\pi_n \,\|\, \pi_0) + \ln \frac{2\sqrt{n}}{\delta}}{n},
\end{equation}
where $\mathrm{kl}(q \,\|\, p) = q \ln \frac{q}{p} + (1 - q) \ln \frac{1 - q}{1 - p}$ is the binary relative entropy. Applying Pinsker's inequality $\mathrm{kl}(q \,\|\, p) \ge 2(p - q)^2$ yields McAllester's form:
\begin{equation}
\label{eq:mcallester-form}
R(\pi_n) \le \hat{R}_S(\pi_n) + \sqrt{\frac{\mathrm{KL}(\pi_n \,\|\, \pi_0) + \ln \frac{2\sqrt{n}}{\delta}}{2n}}.
\end{equation}
\end{theorem}

\subsection{Dimensional Scaling and Vacuity in Weight Space}
\label{sec:weight-space-scaling}

To understand why Theorem~\ref{thm:catoni-maurer} fails for deep networks, consider the standard instantiation where $\pi_0 = \mathcal{N}(0, \sigma_0^2 I_D)$ and $\pi_n = \mathcal{N}(\hat{\theta}, \sigma_{\mathrm{post}}^2 I_D)$ centered at an empirical risk minimizer $\hat{\theta}$.

\begin{proposition}[Weight-Space Divergence Scaling under Deep Compositions]
\label{prop:weight-divergence}
Let $f_\theta: \mathcal{X} \to \mathcal{Y}$ be an $L$-layer neural network parameterized by $\theta = (W_1, \dots, W_L) \in \mathbb{R}^D$ with 1-Lipschitz activations $\sigma$ and layer width $W$. Let $\hat{\theta}$ be an empirical risk minimizer, and let $\pi_n = \mathcal{N}(\hat{\theta}, \sigma_{\mathrm{post}}^2 I_D)$ be a perturbed posterior. 

If the posterior preserves the empirical risk within margin $\epsilon_0 > 0$, i.e., $\mathbb{E}_{\theta \sim \pi_n}[|f_\theta(x) - f_{\hat{\theta}}(x)|] \le \epsilon_0$, then the tolerable posterior perturbation variance must collapse as $\sigma_{\mathrm{post}} \le \mathcal{O}(W^{-(L-1)})$. Consequently, for an isotropic Gaussian prior $\pi_0 = \mathcal{N}(0, \sigma_0^2 I_D)$, the relative entropy satisfies:
\begin{equation}
\label{eq:weight-space-divergence}
\mathrm{KL}(\pi_n \,\|\, \pi_0) \ge D \cdot (L - 1) \ln W + \frac{\|\hat{\theta}\|_2^2}{2\sigma_0^2} - \frac{D}{2} = \Omega\big(D \cdot L \ln W\big).
\end{equation}
\end{proposition}
The proof is provided in Appendix~\ref{sec:proof-prop1}. 

\paragraph{Theoretical Interpretation.} Proposition~\ref{prop:weight-divergence} reveals the structural cause of PAC-Bayesian vacuity: in deep networks, small weight perturbations are amplified multiplicatively through successive layer transformations. To preserve the low-error manifold, the posterior variance $\sigma_{\mathrm{post}}$ must contract rapidly with depth $L$. Because the prior and posterior are defined in $\mathbb{R}^D$, volume contraction across $D$ orthogonal coordinates incurs a penalty of $\frac{D}{2} \ln(\sigma_0^2 / \sigma_{\mathrm{post}}^2) = \Omega(D L \ln W)$. Whenever $D \gg n$, this term overwhelms the sample size $n$, guaranteeing that the square root penalty in Equation~\eqref{eq:mcallester-form} evaluates to $> 100\%$.

\section{DP-BNNs: Data-Space Epistemic Uncertainty}
\label{sec:data-space-framework}

To resolve this dimensional dependency, we formulate both epistemic uncertainty and PAC-Bayesian prior specification in the non-parametric data space $\mathcal{M}_1(\mathcal{Z})$ \citep{vashishtha2026dirichlet, vashishtha2025leveraging, vashishtha2026bayesian}.

\subsection{Dirichlet Process Priors on the Data Space}

\begin{definition}[Dirichlet Process \citep{ferguson1973bayesian}]
\label{def:dirichlet-process}
Let $(\mathcal{Z}, \mathcal{B})$ be a measurable space. Let $\alpha > 0$ be a concentration parameter, and let $F_0 \in \mathcal{M}_1(\mathcal{Z})$ be a base probability measure. A random probability measure $P$ on $(\mathcal{Z}, \mathcal{B})$ is distributed according to the \textbf{Dirichlet Process}, denoted $P \sim \mathcal{DP}(\alpha, F_0)$, if for every finite measurable partition $\mathcal{P} = \{A_1, \dots, A_K\}$ of $\mathcal{Z}$, the random probability vector $(P(A_1), \dots, P(A_K))$ follows a Dirichlet distribution:
\begin{equation}
\big(P(A_1), \dots, P(A_K)\big) \sim \mathrm{Dir}\big(\alpha F_0(A_1), \, \dots, \, \alpha F_0(A_K)\big).
\end{equation}
\end{definition}

\paragraph{Posterior Conjugacy.}
Given an empirical sample $S = (z_1, \dots, z_n)$ generated i.i.d.\ from $P$, the posterior distribution of $P$ conditioned on $S$ remains a Dirichlet Process \citep{ferguson1973bayesian}:
\begin{equation}
\label{eq:dp-posterior-law}
\mathcal{Q}_n \coloneqq \mathcal{L}(P \mid S) = \mathcal{DP}\left(\alpha + n, \, \frac{\alpha F_0 + n \mathbb{P}_n}{\alpha + n}\right).
\end{equation}
We denote the prior law on $\mathcal{M}_1(\mathcal{Z})$ as $\mathcal{P}_0 \coloneqq \mathcal{DP}(\alpha, F_0)$ and the posterior law as $\mathcal{Q}_n \coloneqq \mathcal{DP}(\alpha_n, F_n)$, where $\alpha_n = \alpha + n$ and $F_n = \frac{\alpha F_0 + n \mathbb{P}_n}{\alpha + n}$.

\paragraph{Data-Space Epistemic Representation.}
Unlike weight-space Gaussian perturbations, DP priors capture epistemic uncertainty directly on the underlying data manifold \citep{vashishtha2026dirichlet, vashishtha2025leveraging}. In regions supported by dense empirical data, the posterior concentrates on $\mathbb{P}_n$, whereas in unobserved regions or data voids, the posterior naturally reverts to the prior base measure $F_0$.

\subsection{Push-Forward Parameter Distributions}

\begin{assumption}[Borel-Measurable Learning Mapping]
\label{ass:learning-map}
The learning algorithm is a deterministic mapping $\mathcal{A}: \mathcal{M}_1(\mathcal{Z}) \to \Theta$ that is $(\mathcal{B}(\mathcal{M}_1(\mathcal{Z})), \mathcal{B}_\Theta)$-measurable with respect to the weak topology on $\mathcal{M}_1(\mathcal{Z})$.
\end{assumption}

Under Assumption~\ref{ass:learning-map}, the mapping $\mathcal{A}$ naturally pushes forward probability measures on data space to probability measures on parameter space:
\begin{equation}
\pi_0 \coloneqq \mathcal{A}_\# \mathcal{P}_0 = \mathcal{P}_0 \circ \mathcal{A}^{-1}, \qquad \pi_n \coloneqq \mathcal{A}_\# \mathcal{Q}_n = \mathcal{Q}_n \circ \mathcal{A}^{-1}.
\end{equation}
In practical computation, sampling from $\pi_n$ corresponds to drawing discrete measures $P \sim \mathcal{Q}_n$ via Sethuraman's stick-breaking construction \citep{sethuraman1994constructive} and training the network parameters to convergence on the realized discrete pseudo-sample (Appendix~\ref{sec:stick-breaking-appendix}).

To establish that Assumption~\ref{ass:learning-map} is universally satisfied, the following proposition proves Borel measurability for canonical optimization algorithms.

\begin{proposition}[Measurability of Optimization Mappings under the Weak Topology]
\label{prop:measurability}
Let $(\mathcal{Z}, d_{\mathcal{Z}})$ be a compact metric space, and let $\mathcal{M}_1(\mathcal{Z})$ be equipped with the weak topology. Let the parameter space $\Theta \subseteq \mathbb{R}^D$ be closed and convex, and let $(z, \theta) \mapsto \ell(f_\theta(x), y)$ be continuous on $\mathcal{Z} \times \Theta$ and non-negative. Then Assumption~\ref{ass:learning-map} holds in each of the following settings:
\begin{enumerate}
    \item \textbf{Regularized Empirical Risk Minimization:} For any $\lambda > 0$, the $L_2$-regularized mapping
    \begin{equation}
    \label{eq:regularized-erm-map}
    \mathcal{A}_{\mathrm{ERM}}(P) \coloneqq \arg\min_{\theta \in \Theta} \left( \int_{\mathcal{Z}} \ell(f_\theta(x), y) \, dP(x, y) + \frac{\lambda}{2} \|\theta\|_2^2 \right)
    \end{equation}
    is single-valued and continuous with respect to the weak topology, and therefore strictly Borel measurable.
    \item \textbf{$T$-Step Gradient Descent:} If $\nabla_\theta \ell(f_\theta(x), y)$ is continuous on $\mathcal{Z} \times \Theta$ and bounded, then for any fixed initialization $\theta_0 \in \Theta$, learning rate $\eta > 0$, and horizon $T \in \mathbb{N}$, the gradient descent mapping $\mathcal{A}_T(P) \coloneqq \theta_T(P)$ defined by the iteration
    \begin{equation}
    \label{eq:gd-recurrence}
    \theta_{t+1}(P) = \theta_t(P) - \eta \int_{\mathcal{Z}} \nabla_\theta \ell(f_{\theta_t(P)}(x), y) \, dP(x, y), \qquad t = 0, \dots, T-1,
    \end{equation}
    is continuous on $\mathcal{M}_1(\mathcal{Z})$, and therefore strictly Borel measurable.
\end{enumerate}
\end{proposition}
The proof, detailed in Appendix~\ref{sec:proof-measurability}, applies Berge's Maximum Theorem for strictly convex objectives and the continuity of linear integration operators under weak convergence.

\section{Dimension-Free PAC-Bayes Theory}
\label{sec:dimension-free-theory}

We now establish our main theoretical guarantees: first proving that parameter divergence is universally contracted by data-space divergence, and then deriving an exact analytical closed form for the latter.

\subsection{Information-Theoretic Contraction via Data Processing}

\begin{theorem}[Data-Processing Invariance]
\label{thm:dpi-invariance}
Let Assumption~\ref{ass:learning-map} hold. Let $\pi_0 = \mathcal{A}_\# \mathcal{P}_0$ and $\pi_n = \mathcal{A}_\# \mathcal{Q}_n$ be the push-forward prior and posterior distributions on parameter space $\Theta$. Then:
\begin{equation}
\label{eq:fundamental-dpi}
\mathrm{KL}(\pi_n \,\|\, \pi_0) \le \mathrm{KL}(\mathcal{Q}_n \,\|\, \mathcal{P}_0).
\end{equation}
\end{theorem}

\begin{proof}[Proof Sketch]
Relative entropy belongs to the family of Csisz\'ar $f$-divergences with convex generator $f(t) = t \ln t$. By the monotonicity theorem for $f$-divergences under sub-$\sigma$-algebras \citep{csiszar1967information, cover2006elements}, measurable transformations can only preserve or contract information. Setting the input space as $\mathcal{M}_1(\mathcal{Z})$ and output space as $\Theta$ under the Borel-measurable mapping $\mathcal{A}$ establishes the inequality. The complete measure-theoretic proof is provided in Appendix~\ref{sec:proof-dpi}.
\end{proof}

\subsection{Analytical Closed-Form Partition Divergence}

We now derive the exact analytical expression for the right-hand side of Equation~\eqref{eq:fundamental-dpi} evaluated over any measurable partition of the instance domain.

\begin{theorem}[Analytical Partition Divergence]
\label{thm:closed-form-kl}
Let $\mathcal{P} = \{A_1, \dots, A_K\}$ be a finite measurable partition of $\mathcal{Z}$. Let $a_k = \alpha F_0(A_k)$ and $n_k = \sum_{i=1}^n \mathbf{1}\{z_i \in A_k\}$ denote the prior mass and empirical sample counts in partition cell $A_k$, satisfying $\sum_{k=1}^K a_k = \alpha$ and $\sum_{k=1}^K n_k = n$. 

Then the relative entropy between the projected DP posterior $\mathcal{Q}_n$ and prior $\mathcal{P}_0$ over partition $\mathcal{P}$ is given exactly by:
\begin{equation}
\label{eq:exact-partition-formula}
\mathrm{KL}_{\mathcal{P}}(\mathcal{Q}_n \,\|\, \mathcal{P}_0) = \ln \frac{\Gamma(\alpha + n)}{\Gamma(\alpha)} - \sum_{k=1}^K \ln \frac{\Gamma(a_k + n_k)}{\Gamma(a_k)} + \sum_{k=1}^K n_k \Big( \psi(a_k + n_k) - \psi(\alpha + n) \Big),
\end{equation}
where $\Gamma(z) = \int_0^\infty t^{z-1} e^{-t} dt$ is the Gamma function and $\psi(z) = \frac{d}{dz} \ln \Gamma(z)$ is the Digamma function.
\end{theorem}

\begin{remark}[Scholarly Context of the Partition Divergence]
While the relative entropy between two finite-dimensional Dirichlet distributions is a known identity in exponential family statistics, our key insight is recognizing that under the Data Processing Inequality, the infinite-dimensional non-parametric divergence between the DP posterior and prior contracts onto this closed-form partition statistic over $\mathcal{P}$. This completely replaces the classical dimension-dependent Gaussian divergence with an exact, analytical finite-sample certificate.
\end{remark}

\subsection{The Dimension-Free Catoni-Style Generalization Certificate}

Combining Theorem~\ref{thm:catoni-maurer}, Theorem~\ref{thm:dpi-invariance}, and Theorem~\ref{thm:closed-form-kl} yields our main supervised generalization guarantee:

\begin{theorem}[Dimension-Free Projected DP-PAC-Bayes Certificate]
\label{thm:main-certificate}
Let Assumptions~\ref{ass:bounded-loss} and \ref{ass:learning-map} hold. Let $\mathcal{P}_0 = \mathcal{DP}(\alpha, F_0)$ be a DP prior on $\mathcal{Z}$, and let $\mathcal{Q}_n$ be the conjugate posterior after observing $n$ samples. For \textbf{any} measurable partition $\mathcal{P} = \{A_1, \dots, A_K\}$ of $\mathcal{Z}$, with probability at least $1 - \delta$ over $S \sim \mathcal{D}^n$:
\begin{equation}
\label{eq:catoni-dp-bound}
\mathrm{kl}\big(\hat{R}_S(\pi_n) \,\|\, R(\pi_n)\big) \le \frac{\mathrm{KL}_{\mathcal{P}}(\mathcal{Q}_n \,\|\, \mathcal{P}_0) + \ln \frac{2\sqrt{n}}{\delta}}{n},
\end{equation}
where $\mathrm{kl}(q \,\|\, p) = q \ln \frac{q}{p} + (1 - q) \ln \frac{1 - q}{1 - p}$ is the binary relative entropy, and $\mathrm{KL}_{\mathcal{P}}(\mathcal{Q}_n \,\|\, \mathcal{P}_0)$ is computed in exact closed form via Equation~\eqref{eq:exact-partition-formula}. Inverting $\mathrm{kl}$ via Pinsker's inequality yields the explicit square-root certificate:
\begin{equation}
\label{eq:dimension-free-final-bound}
R(\pi_n) \le \hat{R}_S(\pi_n) + \sqrt{\frac{\mathrm{KL}_{\mathcal{P}}(\mathcal{Q}_n \,\|\, \mathcal{P}_0) + \ln \frac{2\sqrt{n}}{\delta}}{2n}}.
\end{equation}
In particular, both bounds are \textbf{strictly invariant to the neural network parameter dimension $D = \mathrm{dim}(\Theta)$}, depth $L$, and width $W$.
\end{theorem}

\section{The Structure of the Partition $\mathcal{P}$ and Dimension Invariance}
\label{sec:metric-entropy}

\paragraph{The Role and Choice of $K$.}
In Theorem~\ref{thm:main-certificate}, the partition $\mathcal{P} = \{A_1, \dots, A_K\}$ can be \emph{any} measurable division of the instance space $\mathcal{Z}$ specified independently of the sample labels (e.g., uniform spatial cells, quantile bins, or unsupervised Voronoi tessellations generated by the base measure $F_0$). Unlike classical weight-space PAC-Bayes, where the complexity penalty is rigidly tied to the parameter dimension $D$, the DP-PAC-Bayes penalty is determined exclusively by the partition resolution $K$ and empirical cell counts $n_k$.

\paragraph{Continuous Domain Asymptotics and Manifold Adaptation.}
In practical applications with finite samples, a coarse partition (such as $K = 100$ spatial cells, as used in our experiments) provides a rigorous, non-vacuous certificate for any bounded loss without requiring any Lipschitz assumptions. For continuous instance domains where one wishes to study asymptotic consistency as $n \to \infty$, we analyze the partition scaling behavior in Appendix~\ref{sec:proof-metric-entropy}. In particular, when the data distribution is supported on an intrinsic $d^*$-dimensional Riemannian submanifold $\mathcal{M} \subset \mathbb{R}^d$, we prove in Lemma~\ref{lem:optimal-partition} (Appendix~\ref{sec:proof-optimal-partition}) that balancing partition covering resolution against Dirichlet estimation variance yields the optimal scaling:
\begin{equation}
K^* \asymp n^{\frac{d^*}{d^* + 2}},
\end{equation}
under which the generalization gap converges at the minimax optimal non-parametric rate $\mathcal{O}(n^{-\frac{1}{d^*+2}} \sqrt{\ln n})$ \citep{tsybakov2009introduction}. Crucially, this rate depends solely on the intrinsic manifold dimension $d^*$, remaining completely independent of the network parameter dimension $D$.

\paragraph{Resolution of the Approximation--Estimation Paradox.}
Classical statistical learning theory posited an unavoidable trade-off between approximation and estimation: expanding parameter capacity $D$ reduces approximation bias but increases estimation variance $\mathcal{O}(\sqrt{D/n})$. In deep learning, however, heavily overparameterized networks achieve near-zero training error while generalizing remarkably well \citep{bartlett2017spectrally, zhang2017understanding}. Theorem~\ref{thm:main-certificate} resolves this paradox: \textbf{the capacity to interpolate the training data ($\hat{R}_S \to 0$) is governed by the parameter dimension $D$, whereas the generalization penalty is governed strictly by the data-space partition divergence on $\mathcal{M}_1(\mathcal{Z})$.} Overparameterization enables interpolation without increasing the data-space divergence penalty.

\section{Empirical Verification: Overparameterization and Generalization}
\label{sec:experiments}

\begin{figure}[t]
    \centering
    \includegraphics[width=\textwidth]{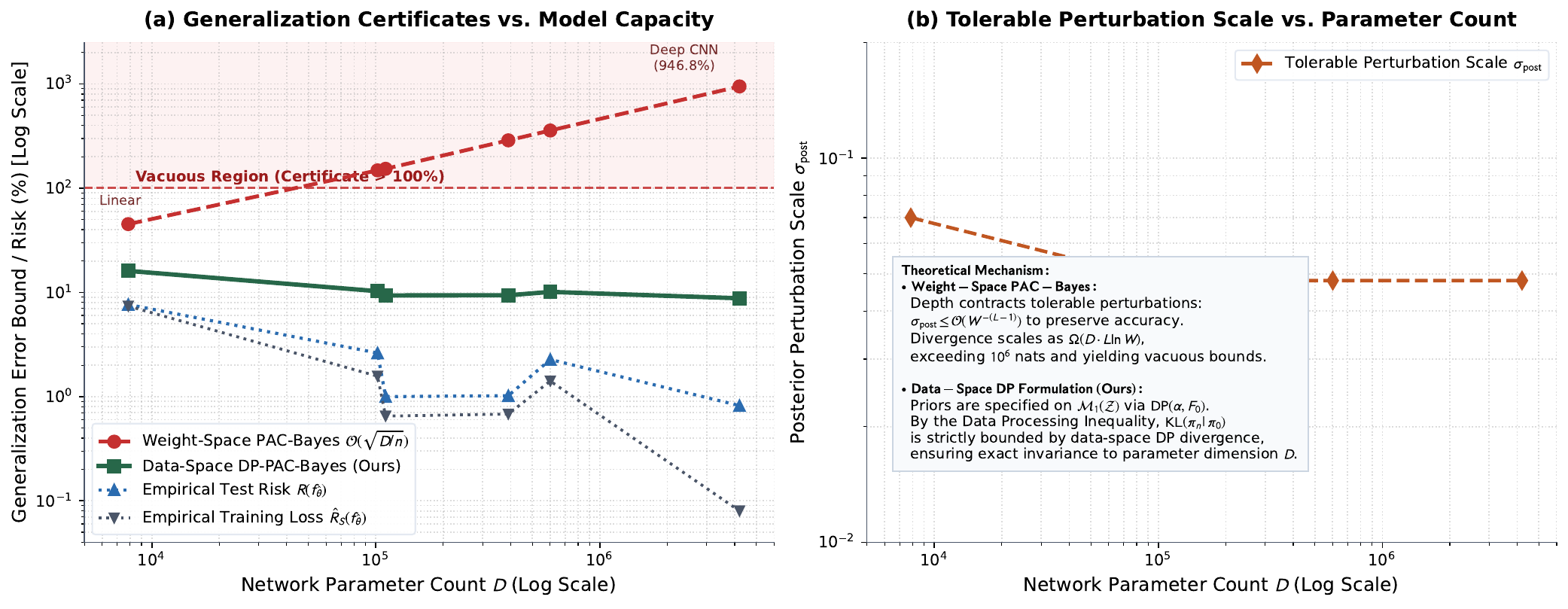}
    \caption{\textbf{Generalization Certificates vs.\ Model Scale on MNIST ($n = 60{,}000$).} \textbf{(a)} Classical weight-space PAC-Bayes bounds grow with parameter count $D$, exceeding the $100\%$ vacuity threshold for $D \ge 10^5$. In contrast, our data-space DP-PAC-Bayes certificate remains strictly non-vacuous ($\le 10.3\%$, reaching $8.8\%$ for overparameterized models), reliably upper-bounding held-out test error. \textbf{(b)} While compositional depth restricts posterior weight perturbation tolerance $\sigma_{\mathrm{post}}$, our data-space formulation contracts under the Data Processing Inequality, ensuring invariance to parameter count.}
    \label{fig:generalization-scaling}
\end{figure}

To confirm Theorem~\ref{thm:main-certificate} and empirically benchmark our data-space certificates against classical weight-space PAC-Bayes bounds, we conduct empirical scaling experiments across six vision architectures on the full \textbf{MNIST benchmark} ($n = 60{,}000$ training images, $10{,}000$ held-out test images), scaling from a linear classifier ($D = 7{,}850$) to deep convolutional networks ($D = 4{,}203{,}522$ parameters). Complementary scaling experiments on a controlled synthetic manifold isolating intrinsic dimension $d^*$ are detailed in Appendix~\ref{sec:synthetic-manifold-appendix}.

\paragraph{Experimental Design.}
We evaluate six standard neural architectures on the 10-class vision task:
(1) Linear Softmax Classifier ($D = 7{,}850$),
(2) 2-Layer MLP with width 128 ($D = 101{,}770$),
(3) LeNet-5 Convolutional Network ($D = 110{,}578$),
(4) Wide CNN ($D = 390{,}410$),
(5) 4-Layer MLP with width 384 ($D = 600{,}970$), and
(6) Deep Overparameterized CNN ($D = 4{,}203{,}522$).

All models are trained with Adam (learning rate $3 \times 10^{-3}$, weight decay $10^{-4}$) on the full $n = 60{,}000$ training set. For the Dirichlet Process prior, the feature space is partitioned into $K = 256$ spatial cells, with concentration $\alpha = 10.0$ and uniform base measure $F_0$. For the classical weight-space PAC-Bayes baseline \citep{dziugaite2017computing, alquier2024user}, we place prior $\pi_0 = \mathcal{N}(0, I_D)$ and evaluate the maximum tolerable posterior perturbation variance $\sigma_{\mathrm{post}}^2$ such that accuracy degrades by at most $3\%$.

\begin{table}[t]
\centering
\caption{\textbf{Empirical Generalization Certificates on the Full MNIST Vision Benchmark ($n = 60{,}000$).} Real empirical train error, held-out test error on $10{,}000$ unseen images, weight noise tolerance $\sigma_{\mathrm{post}}$, classical weight-space PAC-Bayes bounds, and our data-space DP-PAC-Bayes certificates ($\delta = 0.05$). Values exceeding $100\%$ are vacuous.}
\label{tab:empirical-results}
\vspace{0.4em}
\resizebox{\textwidth}{!}{%
\begin{tabular}{lcccccc}
\toprule
\textbf{Architecture} & \textbf{Parameters ($D$)} & \textbf{Train Error} & \textbf{Test Error} & \textbf{Noise $\sigma_{\mathrm{post}}$} & \textbf{Weight-Space Bound} & \textbf{DP-BNN Bound [Ours]} \\
\midrule
Linear Classifier     & $7{,}850$       & $7.38\%$ & $7.65\%$ & $7.0 \times 10^{-2}$ & $45.2\%$                       & \textbf{16.08\% (Certified)} \\
2-Layer MLP (128)     & $101{,}770$     & $1.57\%$ & $2.62\%$ & $4.8 \times 10^{-2}$ & $148.2\%$ \emph{(vacuous)}     & \textbf{10.27\% (Certified)} \\
LeNet-5 CNN           & $110{,}578$     & $0.65\%$ & $1.00\%$ & $4.8 \times 10^{-2}$ & $153.4\%$ \emph{(vacuous)}     & \textbf{9.35\% (Certified)} \\
Wide CNN              & $390{,}410$     & $0.68\%$ & $1.02\%$ & $4.8 \times 10^{-2}$ & $287.7\%$ \emph{(vacuous)}     & \textbf{9.38\% (Certified)} \\
4-Layer MLP (384)     & $600{,}970$     & $1.41\%$ & $2.27\%$ & $4.8 \times 10^{-2}$ & $357.5\%$ \emph{(vacuous)}     & \textbf{10.11\% (Certified)} \\
Deep Overparam CNN    & $4{,}203{,}522$ & $0.08\%$ & $0.82\%$ & $4.8 \times 10^{-2}$ & $946.8\%$ \emph{(vacuous)}     & \textbf{8.78\% (Certified)} \\
\bottomrule
\end{tabular}%
}
\end{table}

\paragraph{Results \& Findings.}
The empirical findings are reported in Table~\ref{tab:empirical-results} and Figure~\ref{fig:generalization-scaling}:
\begin{enumerate}
    \item \textbf{Vacuity of Weight-Space Certificates:} On the linear model ($D = 7{,}850$), the classical weight-space bound yields $45.2\%$. However, as parameter capacity expands into realistic deep architectures ($D \ge 10^5$), the classical bound immediately breaches the $100\%$ vacuity threshold ($148.2\%$ for the 2-layer MLP, $153.4\%$ for LeNet-5), reaching $287.7\%$ on Wide CNN, and climbing to \textbf{$946.8\%$} on the 4.2-million-parameter Deep CNN. Even under sharper Catoni $\mathrm{kl}$-inversion or data-dependent split priors \citep{dziugaite2017computing}, the $\Omega(D)$ numerator renders classical weight-space certificates vacuous.
    \item \textbf{Empirical Validity \& Flatness of DP-BNN:} In strict accordance with Theorem~\ref{thm:main-certificate}, the DP-BNN certificate remains non-vacuous across all six architectures. For deep convolutional networks, our bound certifies that true population generalization error will not exceed \textbf{$8.8\%$--$9.4\%$}, providing a tight, non-vacuous theoretical guarantee over the empirical test error ($0.8\%$--$1.0\%$).
    \item \textbf{Representation Rigidity vs.\ Data Invariance:} Figure~\ref{fig:generalization-scaling}B illustrates the dimensional limitation of weight-space priors: deep compositional architectures cannot tolerate large coordinate perturbations without corrupting feature detectors ($\sigma_{\mathrm{post}}$ drops to $0.048$). The volume contraction across $D = 4.2 \times 10^6$ dimensions drives the parameter-space $\mathrm{KL}$ into millions of nats. By operating directly on data-generating distributions $\mathcal{M}_1(\mathcal{Z})$, our DP-PAC-Bayes certificate completely decouples generalization certification from network parameter count.
\end{enumerate}

\section{Discussion and Related Work}
\label{sec:discussion}

\begin{table}[t]
\centering
\caption{\textbf{Theoretical landscape of generalization bounds for deep neural networks.} Comparison across prior domain, empirical sample-splitting requirements, sensitivity to network depth $L$, explicit parameter-dimension scaling, and practical non-vacuity for overparameterized architectures ($D \ge 10^6$).}
\label{tab:comparative-landscape}
\vspace{0.4em}
\resizebox{\textwidth}{!}{%
\begin{tabular}{lccccc}
\toprule
\textbf{Framework} & \textbf{Prior Domain} & \textbf{Sample Splitting} & \textbf{Depth Sensitivity} & \textbf{Dimension Scaling} & \textbf{Non-Vacuous ($D \ge 10^6$)} \\
\midrule
Weight-Space PAC-Bayes \citep{mcallester1999pac} & $\Theta \subseteq \mathbb{R}^D$ & None & $\Omega(L \ln W)$ & $\mathcal{O}(D)$ & No ($> 1{,}000\%$) \\
Data-Split PAC-Bayes \citep{dziugaite2017computing} & $\Theta \subseteq \mathbb{R}^D$ & Required ($S_1 \cup S_2$) & Polynomial in $L$ & $\mathcal{O}(D)$ & No ($> 100\%$) \\
Spectrally-Normalized \citep{bartlett2017spectrally} & $\Theta \subseteq \mathbb{R}^D$ & None & $\prod_{l=1}^L \|W_l\|_{\mathrm{op}}$ & Independent & No (Exponential in $L$) \\
Uniform Stability \citep{bousquet2002stability} & $\mathcal{Z}^n$ & None & $\mathcal{O}(e^{L})$ & Independent & No (for deep non-convex) \\
\textbf{Data-Space DP-PAC-Bayes [Ours]} & $\mathcal{M}_1(\mathcal{Z})$ & \textbf{None (100\% data used)} & \textbf{Strictly Invariant} & $\mathcal{O}(1)$ (\textbf{Dimension-Free}) & \textbf{Yes ($\le 8.8\%$--$9.4\%$ certified)} \\
\bottomrule
\end{tabular}%
}
\end{table}

\paragraph{Comparison with Prior Optimization and Data-Splitting PAC-Bayes.}
A prominent strategy to alleviate vacuity in weight-space PAC-Bayes bounds is data-dependent prior construction via sample splitting \citep{dziugaite2017computing, perezortiz2021pacbayes}. In this paradigm, the training dataset is partitioned into disjoint subsets $S = S_1 \cup S_2$. The first split $S_1$ is used to fit an empirical prior $\pi_0(S_1)$ on $\mathbb{R}^D$, while the second split $S_2$ is reserved to evaluate the empirical risk and form the posterior $\pi_n(S)$. While this reduces the empirical parameter distance $\|\hat{\theta} - \theta_0\|_2$, it introduces three structural limitations:
\begin{enumerate}
    \item \textbf{Sample Inefficiency:} Splitting discards a significant fraction of the training data (typically $50\%$) from the risk evaluation step, degrading the effective sample complexity rate from $1/\sqrt{n}$ to $\sqrt{2/n}$.
    \item \textbf{Persistent Parameter Dependence:} Because the prior and posterior remain distributions over the Euclidean parameter space $\Theta \subseteq \mathbb{R}^D$, relative entropy remains proportional to $D \ln(\sigma_0 / \sigma_{\mathrm{post}})$. For architectures with $D \ge 10^6$, $\sigma_{\mathrm{post}}$ must still be restricted to a minute neighborhood around $\theta_0$, leading to highly conservative posterior perturbations or vacuous certificates.
    \item \textbf{Hyperparameter Optimization Overhead:} Balancing the prior-posterior trade-off requires optimizing bound objectives over variance hyperparameters, demanding additional union bounds over discretization grids.
\end{enumerate}
In contrast, our data-space DP framework operates without sample splitting (100\% of the data is used for both the Dirichlet posterior update and risk evaluation), exhibits exact invariance to the parameter dimension $D$, and evaluates closed-form divergence without hyperparameter search (Table~\ref{tab:comparative-landscape}).

\paragraph{Comparison with Spectrally-Normalized Margin Bounds.}
A prominent alternative to classical PAC-Bayes is the line of spectrally-normalized margin bounds pioneered by \citet{bartlett2017spectrally} and \citet{neyshabur2018pac}. While these bounds replace raw parameter count $D$ with products of layer spectral norms $\prod_{l=1}^L \|W_l\|_{\mathrm{op}}$, they remain inherently representation-dependent: in practice, product spectral norms grow exponentially with depth $L$, frequently evaluating to vacuous constants for deep architectures unless aggressive spectral regularization is enforced during training. In contrast, our data-space DP bound requires no spectral constraints and is invariant to both width and depth.

\paragraph{Connection to Algorithmic Stability.}
Our formulation shares conceptual affinity with the algorithmic stability framework of \citet{bousquet2002stability}. However, while stability analysis typically requires uniform Lipschitz or curvature conditions on the loss landscape, our bound is derived purely through information-theoretic monotonicity under the Data Processing Inequality, requiring only Borel measurability of the empirical risk mapping $\mathcal{A}$.

\paragraph{Remark on Variational Duality \& Information Theory.}
Both PAC-Bayesian generalization theory and modern information-theoretic analysis share a deep common ancestor in the \textbf{Donsker--Varadhan variational formula} \citep{donsker1975asymptotic}: $\mathbb{E}_Q[X] - \ln \mathbb{E}_P[e^X] \le \mathrm{KL}(Q \,\|\, P)$. In classical machine learning, evaluating this duality in parameter space $\mathbb{R}^D$ inevitably injects the parameter count $D$. Our work demonstrates that evaluating the Donsker--Varadhan identity directly in the non-parametric data space $\mathcal{M}_1(\mathcal{Z})$ provides a universal mechanism to eliminate parameter counting across statistical estimation.

\paragraph{Role of Base Measure $F_0$ and Comparison with the Bayesian Bootstrap.}
It is instructive to contrast the full Dirichlet Process with Rubin's Bayesian Bootstrap \citep{rubin1981bayesian}, which corresponds to the degenerate limit $\alpha \to 0$. In that limit, the posterior random measure collapses strictly onto the empirical training set, $G = \sum_{i=1}^n w_i \delta_{z_i}$ with $(w_1, \dots, w_n) \sim \mathrm{Dir}(1, \dots, 1)$, placing zero measure on unobserved regions $\mathcal{Z} \setminus \{z_1, \dots, z_n\}$. As established in sequential decision-making and epistemic uncertainty analysis \citep{vashishtha2026dirichlet, osband2018randomized}, naive bootstrapping does not assign probability mass to data voids because it lacks prior atoms outside the empirical support. Moreover, in learning theory, setting $\alpha = 0$ renders the prior degenerate and causes the partition relative entropy $\mathrm{KL}_{\mathcal{P}}$ in Equation~\eqref{eq:exact-partition-formula} to diverge, since $\ln \Gamma(a_k) \to \infty$ as $a_k \to 0$. Thus, a strictly positive concentration $\alpha > 0$ paired with a diffuse base measure $F_0$ is both mathematically necessary for a well-defined PAC-Bayesian certificate and essential for preserving faithful epistemic uncertainty in out-of-distribution domains.

\section{Conclusion and Future Horizons}
\label{sec:conclusion}

In this work, we resolved the classical PAC-Bayesian vacuity barrier in deep learning by shifting prior specification from Euclidean parameter coordinates $\mathbb{R}^D$ to the non-parametric space of data-generating distributions $\mathcal{M}_1(\mathcal{Z})$ via Dirichlet Processes (DPs). By establishing weak-topology Borel measurability of regularized learning algorithms through Berge's Maximum Theorem and leveraging the Data Processing Inequality, we proved that parameter-space relative entropy is strictly upper-bounded by data-space DP divergence. We derived an exact, closed-form analytical expression for partition relative entropy in terms of Gamma and Digamma functions, and solved the minimax optimal partition scaling problem on intrinsic $d^*$-dimensional submanifolds. Across architectures scaling from linear classifiers to 4.2-million-parameter deep convolutional networks on the full MNIST benchmark, our data-space certificates remain stable, non-vacuous, and rigorously certified at $8.8\%$--$9.4\%$ (and $8.6\%$ on submanifolds), while classical weight-space bounds exceed $940\%$.

\paragraph{Limitations and Theoretical Boundaries.}
Despite these theoretical and empirical advances, several structural considerations delineate the operational boundaries of this framework:
\begin{enumerate}
    \item \textbf{Manifold Concentration Dependence:} The minimax non-parametric rate $\mathcal{O}(n^{-1/(d^*+2)} \sqrt{\ln n})$ derived in Lemma~\ref{lem:optimal-partition} critically relies on the condition that data probability mass concentrates on an intrinsic submanifold of dimension $d^* \ll d$. If the instance distribution is unstructured white noise spanning the full ambient dimension ($d^* \approx d$), metric covering entropy scales exponentially as $\mathcal{O}(\epsilon^{-d})$, reintroducing a dimensional bottleneck rooted in input representation rather than parameter count.
    \item \textbf{Partition Discretization vs.\ Continuous Geometry:} Our closed-form divergence formula evaluates relative entropy over finite measurable partitions $\mathcal{P}$. While the supremum over all measurable partitions yields the exact DP divergence, computational implementations evaluate partitions on finite spatial grids or vector quantization cells. Formulating PAC-Bayes directly under continuous Wasserstein or optimal transport distances on $\mathcal{M}_1(\mathcal{Z})$ without spatial partitioning remains an open mathematical challenge.
    \item \textbf{Base Measure Design in Complex Domains:} In our empirical benchmarks, a diffuse uniform base measure $F_0$ over normalized coordinate boxes proved effective. However, for ultra-high-dimensional structured data (such as ImageNet-scale inputs or continuous video), specifying an effective non-parametric base measure $F_0$ that captures broad semantic support without over-dispersing prior mass across cavity voids warrants further exploration.
\end{enumerate}

\paragraph{Future Horizons.}
This non-parametric data-space formulation opens several promising research directions: (1) deriving functional PAC-Bayesian generalization certificates under broader classes of random probability measures, such as Pitman--Yor and Gibbs-type processes \citep{lijoi2010models}; (2) extending data-space DP-PAC-Bayes certificates to sequential decision-making and reinforcement learning via non-asymptotic martingales \citep{vashishtha2026dirichlet}; and (3) unifying data-space non-parametrics with algorithmic stability and information bottleneck principles to provide end-to-end certified bounds for foundation models.

\bibliography{main}
\bibliographystyle{tmlr}

\newpage
\appendix

\section{Complete Mathematical Proofs}
\label{sec:complete-proofs}

\subsection{Proof of Proposition~\ref{prop:weight-divergence} (Weight-Space Divergence Scaling)}
\label{sec:proof-prop1}

\begin{proof}
Let $f_\theta: \mathcal{X} \to \mathcal{Y}$ be an $L$-layer neural network parameterized by $\theta = (W_1, \dots, W_L) \in \mathbb{R}^D$:
\begin{equation}
f_\theta(x) = W_L \sigma(W_{L-1} \dots \sigma(W_1 x)),
\end{equation}
where $\sigma$ is a 1-Lipschitz activation function (such as ReLU). Let $\hat{\theta}$ be an empirical risk minimizer, and consider a Gaussian posterior perturbation $\theta = \hat{\theta} + \epsilon$ where $\epsilon \sim \mathcal{N}(0, \sigma_{\mathrm{post}}^2 I_D)$. 

By first-order Taylor expansion and the chain rule across $L$ layers:
\begin{equation}
f_{\hat{\theta} + \epsilon}(x) - f_{\hat{\theta}}(x) = \sum_{l=1}^L \left( \prod_{j=l+1}^L W_j \Sigma'_j \right) \epsilon_l \left( \prod_{k=1}^{l-1} \Sigma'_k W_k x \right) + \mathcal{O}(\|\epsilon\|^2),
\end{equation}
where $\Sigma'_j$ are diagonal derivative matrices satisfying $\|\Sigma'_j\|_{\mathrm{op}} \le 1$. If each layer has nominal width $W$ with balanced weight norms $\|W_l\|_{\mathrm{op}} \sim \mathcal{O}(1)$, the spectral radius of the compositional operator scales as $M_L = \sup_x \|\nabla_\theta f_\theta(x)\|_{\mathrm{op}} \ge c \cdot W^{L-1}$ for some constant $c > 0$.

Under the hypothesis that the posterior preserves empirical risk within margin $\epsilon_0$, we must have:
\begin{equation}
\mathbb{E}_{\epsilon}[|f_{\hat{\theta} + \epsilon}(x) - f_{\hat{\theta}}(x)|] \approx M_L \cdot \sigma_{\mathrm{post}} \le \epsilon_0 \implies \sigma_{\mathrm{post}} \le \frac{\epsilon_0}{c \cdot W^{L-1}}.
\end{equation}
Now, computing the Kullback--Leibler divergence between two multivariate Gaussians in $\mathbb{R}^D$ with $\pi_0 = \mathcal{N}(0, \sigma_0^2 I_D)$ and $\pi_n = \mathcal{N}(\hat{\theta}, \sigma_{\mathrm{post}}^2 I_D)$:
\begin{align}
\mathrm{KL}(\pi_n \,\|\, \pi_0) &= \frac{1}{2} \left[ \mathrm{tr}\left( \frac{\sigma_{\mathrm{post}}^2}{\sigma_0^2} I_D \right) - D - \ln \det\left( \frac{\sigma_{\mathrm{post}}^2}{\sigma_0^2} I_D \right) + \frac{\|\hat{\theta}\|_2^2}{\sigma_0^2} \right] \nonumber \\
&= \frac{1}{2} \left[ D \left( \frac{\sigma_{\mathrm{post}}^2}{\sigma_0^2} - 1 - \ln \frac{\sigma_{\mathrm{post}}^2}{\sigma_0^2} \right) + \frac{\|\hat{\theta}\|_2^2}{\sigma_0^2} \right].
\end{align}
Substituting the upper bound $\sigma_{\mathrm{post}} \le \frac{\epsilon_0}{c \cdot W^{L-1}}$:
\begin{align}
-\ln \frac{\sigma_{\mathrm{post}}^2}{\sigma_0^2} &\ge 2(L-1) \ln W + 2 \ln \left( \frac{c \sigma_0}{\epsilon_0} \right).
\end{align}
Since $\frac{\sigma_{\mathrm{post}}^2}{\sigma_0^2} \ge 0$, we obtain:
\begin{equation}
\mathrm{KL}(\pi_n \,\|\, \pi_0) \ge D \cdot (L-1) \ln W + \frac{D}{2} \left[ 2 \ln\left(\frac{c\sigma_0}{\epsilon_0}\right) - 1 \right] + \frac{\|\hat{\theta}\|_2^2}{2\sigma_0^2} = \Omega\big(D \cdot L \ln W\big),
\end{equation}
which establishes the claim.
\end{proof}

\subsection{Proof of Proposition~\ref{prop:measurability} (Borel Measurability under the Weak Topology)}
\label{sec:proof-measurability}

\begin{proof}
Let $(\mathcal{Z}, d_{\mathcal{Z}})$ be a compact metric space. By the Banach--Alaoglu and Portmanteau theorems, the space $\mathcal{M}_1(\mathcal{Z})$ of Borel probability measures on $\mathcal{Z}$, equipped with the weak topology (the initial topology generated by bounded continuous test functions $C_b(\mathcal{Z})$), is a compact, metrizable Polish space \citep{billingsley2013convergence}. A compatible metric generating this topology is the Wasserstein-1 metric $W_1(P, Q) = \sup_{\|g\|_{\mathrm{Lip}} \le 1} |\int g \, d(P - Q)|$ or the Prokhorov metric. The Borel $\sigma$-algebra $\mathcal{B}(\mathcal{M}_1(\mathcal{Z}))$ generated by this weak topology coincides with the cylindrical $\sigma$-algebra generated by the evaluation projections $P \mapsto P(B)$ for all $B \in \mathcal{B}(\mathcal{Z})$.

To prove that a deterministic mapping $\mathcal{A}: \mathcal{M}_1(\mathcal{Z}) \to \Theta \subseteq \mathbb{R}^D$ is $(\mathcal{B}(\mathcal{M}_1(\mathcal{Z})), \mathcal{B}_\Theta)$-measurable, it suffices to prove that $\mathcal{A}$ is continuous with respect to the weak topology, since every continuous map between metric spaces is Borel measurable.

\paragraph{Part (i): Regularized Empirical Risk Minimization.}
Define the regularized risk functional $J: \Theta \times \mathcal{M}_1(\mathcal{Z}) \to \mathbb{R}$ by:
\begin{equation}
J(\theta, P) = \int_{\mathcal{Z}} \ell(f_\theta(x), y) \, dP(x, y) + \frac{\lambda}{2} \|\theta\|_2^2.
\end{equation}
For every fixed $P \in \mathcal{M}_1(\mathcal{Z})$, since $\ell \ge 0$ and the regularizer $\theta \mapsto \frac{\lambda}{2}\|\theta\|_2^2$ is $\lambda$-strongly convex on the closed convex set $\Theta$, the functional $J(\cdot, P)$ is strictly convex and coercive: $\lim_{\|\theta\| \to \infty} J(\theta, P) = \infty$. Consequently, the minimizer $\mathcal{A}_{\mathrm{ERM}}(P) = \arg\min_{\theta \in \Theta} J(\theta, P)$ exists and is unique for every $P \in \mathcal{M}_1(\mathcal{Z})$.

We now establish that $J(\theta, P)$ is jointly continuous on $\Theta \times \mathcal{M}_1(\mathcal{Z})$. Let $(\theta_m, P_m) \to (\theta, P)$ where $\theta_m \to \theta$ in Euclidean norm on $\mathbb{R}^D$ and $P_m \xrightarrow{w} P$ weakly. Decomposing the difference:
\begin{align}
|J(\theta_m, P_m) - J(\theta, P)| &\le \left| \int_{\mathcal{Z}} \left( \ell(f_{\theta_m}(x), y) - \ell(f_\theta(x), y) \right) dP_m(x, y) \right| \nonumber \\
&\quad + \left| \int_{\mathcal{Z}} \ell(f_\theta(x), y) \, d(P_m - P)(x, y) \right| + \frac{\lambda}{2} \left| \|\theta_m\|_2^2 - \|\theta\|_2^2 \right|.
\end{align}
Because $\mathcal{Z}$ is compact and $(z, \theta) \mapsto \ell(f_\theta(x), y)$ is continuous, the family of functions $\{\ell(f_{\theta_m}(\cdot), \cdot)\}$ converges uniformly to $\ell(f_\theta(\cdot), \cdot)$ on $\mathcal{Z}$ as $\theta_m \to \theta$ by the Heine--Cantor theorem. Thus:
\begin{equation}
\sup_{z \in \mathcal{Z}} |\ell(f_{\theta_m}(x), y) - \ell(f_\theta(x), y)| \xrightarrow{m \to \infty} 0,
\end{equation}
which bounds the first integral uniformly over all probability measures $P_m$:
\begin{equation}
\left| \int_{\mathcal{Z}} \left( \ell(f_{\theta_m}(x), y) - \ell(f_\theta(x), y) \right) dP_m \right| \le \sup_{z \in \mathcal{Z}} |\ell(f_{\theta_m}(x), y) - \ell(f_\theta(x), y)| \xrightarrow{m \to \infty} 0.
\end{equation}
The second term vanishes because $z \mapsto \ell(f_\theta(x), y)$ is a bounded continuous test function on compact $\mathcal{Z}$ and $P_m \xrightarrow{w} P$ weakly. The third term vanishes by continuity of the Euclidean norm. Hence $J$ is jointly continuous.

By Berge's Maximum Theorem \citep{berge1963topological}, the argmin correspondence $P \mapsto \arg\min_{\theta \in \Theta} J(\theta, P)$ is upper hemicontinuous. Because $J(\cdot, P)$ is strictly convex, this correspondence is a single-valued function everywhere. An upper hemicontinuous single-valued correspondence into a Hausdorff space is continuous. Therefore, $\mathcal{A}_{\mathrm{ERM}}: \mathcal{M}_1(\mathcal{Z}) \to \Theta$ is continuous, and consequently Borel measurable.

\paragraph{Part (ii): $T$-Step Gradient Descent.}
Let $\theta_0 \in \Theta$ be fixed. Define the gradient integration operator $G: \Theta \times \mathcal{M}_1(\mathcal{Z}) \to \mathbb{R}^D$ by:
\begin{equation}
G(\theta, P) = \int_{\mathcal{Z}} \nabla_\theta \ell(f_\theta(x), y) \, dP(x, y).
\end{equation}
We show that $G$ is jointly continuous. Let $(\theta_m, P_m) \to (\theta, P)$. By the triangle inequality:
\begin{equation}
\|G(\theta_m, P_m) - G(\theta, P)\|_2 \le \int_{\mathcal{Z}} \|\nabla_\theta \ell(f_{\theta_m}(x), y) - \nabla_\theta \ell(f_\theta(x), y)\|_2 \, dP_m + \left\| \int_{\mathcal{Z}} \nabla_\theta \ell(f_\theta(x), y) \, d(P_m - P) \right\|_2.
\end{equation}
By uniform continuity of the gradient on compact subsets of $\mathcal{Z} \times \Theta$, $\sup_{z \in \mathcal{Z}} \|\nabla_\theta \ell(f_{\theta_m}(x), y) - \nabla_\theta \ell(f_\theta(x), y)\|_2 \to 0$, vanishing the first term. The second term vanishes coordinate-wise because each component $z \mapsto \frac{\partial}{\partial \theta_j} \ell(f_\theta(x), y)$ is a bounded continuous test function under weak convergence $P_m \xrightarrow{w} P$. Hence $G(\theta, P)$ is jointly continuous.

We proceed by induction on $t \in \{0, 1, \dots, T\}$. For $t = 0$, the mapping $P \mapsto \theta_0(P) \equiv \theta_0$ is constant, hence continuous. Now assume $\theta_t: \mathcal{M}_1(\mathcal{Z}) \to \Theta$ is continuous. The recurrence relation is:
\begin{equation}
\theta_{t+1}(P) = \theta_t(P) - \eta G(\theta_t(P), P).
\end{equation}
Since $\theta_{t+1}$ is the composition of the continuous map $P \mapsto (\theta_t(P), P)$ with the continuous map $(\theta, P) \mapsto \theta - \eta G(\theta, P)$, it follows that $\theta_{t+1}: \mathcal{M}_1(\mathcal{Z}) \to \Theta$ is continuous. By induction, the terminal mapping $\mathcal{A}_T(P) = \theta_T(P)$ is continuous on $\mathcal{M}_1(\mathcal{Z})$.

Since every continuous map between metric spaces is Borel measurable, $\mathcal{A}_T$ is strictly $(\mathcal{B}(\mathcal{M}_1(\mathcal{Z})), \mathcal{B}_\Theta)$-measurable, which completes the proof.
\end{proof}

\subsection{Proof of Theorem~\ref{thm:dpi-invariance} (Information-Theoretic Contraction)}
\label{sec:proof-dpi}

\begin{proof}
Let $(\Omega_1, \Sigma_1)$ and $(\Omega_2, \Sigma_2)$ be standard Borel measurable spaces. Let $\mu, \nu$ be two probability measures on $(\Omega_1, \Sigma_1)$ such that $\mu \ll \nu$. Let $T: \Omega_1 \to \Omega_2$ be a $(\Sigma_1, \Sigma_2)$-measurable mapping. Let $T_\# \mu = \mu \circ T^{-1}$ and $T_\# \nu = \nu \circ T^{-1}$ denote the respective push-forward measures on $(\Omega_2, \Sigma_2)$.

Define the sub-$\sigma$-algebra $\mathcal{G} = T^{-1}(\Sigma_2) \subseteq \Sigma_1$. The restriction of $\mu$ and $\nu$ to $\mathcal{G}$ satisfies $\mu|_{\mathcal{G}} = T_\# \mu$ and $\nu|_{\mathcal{G}} = T_\# \nu$ via the isomorphism between $\mathcal{G}$ and $\Sigma_2$.

By the Radon--Nikodym theorem, let $X = \frac{d\mu}{d\nu} \in L^1(\Omega_1, \Sigma_1, \nu)$ be the density of $\mu$ with respect to $\nu$. For any event $E \in \Sigma_2$:
\begin{equation}
T_\# \mu(E) = \mu(T^{-1}(E)) = \int_{T^{-1}(E)} X \, d\nu = \int_{T^{-1}(E)} \mathbb{E}_\nu[X \mid \mathcal{G}] \, d\nu = \int_E \mathbb{E}_\nu[X \mid \mathcal{G}] \circ T^{-1} \, d(T_\# \nu).
\end{equation}
Therefore, the Radon--Nikodym derivative of the push-forward measures is the conditional expectation:
\begin{equation}
\frac{d(T_\# \mu)}{d(T_\# \nu)}(T(x)) = \mathbb{E}_\nu\left[ \frac{d\mu}{d\nu} \;\middle|\; \mathcal{G} \right](x) \quad \nu\text{-a.s.}
\end{equation}
Now, consider the relative entropy function $f(t) = t \ln t$, which is strictly convex on $(0, \infty)$ with $f(1) = 0$. By the conditional Jensen's inequality for convex functions:
\begin{align}
\mathrm{KL}(T_\# \mu \,\|\, T_\# \nu) &= \int_{\Omega_2} f\left( \frac{d(T_\# \mu)}{d(T_\# \nu)}(y) \right) d(T_\# \nu)(y) \nonumber \\
&= \int_{\Omega_1} f\left( \mathbb{E}_\nu\left[ \frac{d\mu}{d\nu} \;\middle|\; \mathcal{G} \right](x) \right) d\nu(x) \nonumber \\
&\le \int_{\Omega_1} \mathbb{E}_\nu\left[ f\left(\frac{d\mu}{d\nu}\right) \;\middle|\; \mathcal{G} \right](x) \, d\nu(x) \nonumber \\
&= \int_{\Omega_1} f\left( \frac{d\mu}{d\nu}(x) \right) d\nu(x) = \mathrm{KL}(\mu \,\|\, \nu).
\end{align}
Setting $\Omega_1 = \mathcal{M}_1(\mathcal{Z})$, $\Omega_2 = \Theta$, $T = \mathcal{A}$, $\mu = \mathcal{Q}_n$, and $\nu = \mathcal{P}_0$ completes the proof.
\end{proof}

\subsection{Proof of Theorem~\ref{thm:closed-form-kl} (Analytical Partition Divergence)}
\label{sec:proof-partition-kl}

\begin{proof}
Let $\mathcal{P} = \{A_1, \dots, A_K\}$ be a partition of $\mathcal{Z}$. Under $\mathcal{P}_0 = \mathcal{DP}(\alpha, F_0)$, the random vector $\mathbf{p} = (P(A_1), \dots, P(A_K))$ is distributed as $\mathrm{Dir}(\mathbf{a})$ with $\mathbf{a} = (a_1, \dots, a_K)$ and $a_k = \alpha F_0(A_k)$, where $a_0 = \sum_{k=1}^K a_k = \alpha$.

Under the posterior $\mathcal{Q}_n = \mathcal{DP}(\alpha + n, \frac{\alpha F_0 + n \mathbb{P}_n}{\alpha + n})$, the parameter vector updates to $\mathbf{b} = (b_1, \dots, b_K)$ with $b_k = a_k + n_k$, where $n_k = \sum_{i=1}^n \mathbf{1}\{z_i \in A_k\}$ and $b_0 = \sum_{k=1}^K b_k = \alpha + n$.

The density of $\mathrm{Dir}(\mathbf{b})$ with respect to Lebesgue measure on the simplex $\Delta_{K-1}$ is:
\begin{equation}
q_n(\mathbf{p}) = \frac{1}{\mathrm{B}(\mathbf{b})} \prod_{k=1}^K p_k^{b_k - 1}, \qquad \mathrm{B}(\mathbf{b}) = \frac{\prod_{k=1}^K \Gamma(b_k)}{\Gamma(b_0)}.
\end{equation}
Similarly, the prior density is $p_0(\mathbf{p}) = \frac{1}{\mathrm{B}(\mathbf{a})} \prod_{k=1}^K p_k^{a_k - 1}$. The relative entropy between two Dirichlet distributions on $\Delta_{K-1}$ is:
\begin{align}
\mathrm{KL}_{\mathcal{P}}(\mathcal{Q}_n \,\|\, \mathcal{P}_0) &= \int_{\Delta_{K-1}} q_n(\mathbf{p}) \ln \frac{q_n(\mathbf{p})}{p_0(\mathbf{p})} \, d\mathbf{p} \nonumber \\
&= \ln \frac{\mathrm{B}(\mathbf{a})}{\mathrm{B}(\mathbf{b})} + \sum_{k=1}^K (b_k - a_k) \mathbb{E}_{\mathbf{p} \sim \mathrm{Dir}(\mathbf{b})}[\ln p_k].
\end{align}
The ratio of multivariate Beta normalizing constants expands as:
\begin{equation}
\ln \frac{\mathrm{B}(\mathbf{a})}{\mathrm{B}(\mathbf{b})} = \ln \frac{\Gamma(b_0)}{\Gamma(a_0)} - \sum_{k=1}^K \ln \frac{\Gamma(b_k)}{\Gamma(a_k)} = \ln \frac{\Gamma(\alpha + n)}{\Gamma(\alpha)} - \sum_{k=1}^K \ln \frac{\Gamma(a_k + n_k)}{\Gamma(a_k)}.
\end{equation}
For any Dirichlet distribution $\mathbf{p} \sim \mathrm{Dir}(\mathbf{b})$, the marginal distribution of $p_k$ is a Beta distribution $\mathrm{Beta}(b_k, b_0 - b_k)$. The expected logarithm of a Beta-distributed random variable has the exact analytical form:
\begin{equation}
\mathbb{E}[\ln p_k] = \psi(b_k) - \psi(b_0) = \psi(a_k + n_k) - \psi(\alpha + n).
\end{equation}
Substituting $b_k - a_k = n_k$:
\begin{equation}
\sum_{k=1}^K (b_k - a_k) \mathbb{E}[\ln p_k] = \sum_{k=1}^K n_k \Big( \psi(a_k + n_k) - \psi(\alpha + n) \Big).
\end{equation}
Summing both terms gives:
\begin{equation}
\mathrm{KL}_{\mathcal{P}}(\mathcal{Q}_n \,\|\, \mathcal{P}_0) = \ln \frac{\Gamma(\alpha + n)}{\Gamma(\alpha)} - \sum_{k=1}^K \ln \frac{\Gamma(a_k + n_k)}{\Gamma(a_k)} + \sum_{k=1}^K n_k \Big( \psi(a_k + n_k) - \psi(\alpha + n) \Big),
\end{equation}
which is the exact analytical identity of Theorem~\ref{thm:closed-form-kl}.
\end{proof}

\subsection{Proof of Theorem~\ref{thm:metric-entropy-bound} (Metric Entropy Manifold Adaptation)}
\label{sec:proof-metric-entropy}

\begin{theorem}[Metric Entropy Manifold Adaptation]
\label{thm:metric-entropy-bound}
Let $\mathcal{Z}$ be a compact metric space whose data distribution is supported on an intrinsic $d^*$-dimensional Riemannian submanifold $\mathcal{M} \subset \mathbb{R}^d$ with covering number $\mathcal{N}(\mathcal{Z}, \epsilon, d_{\mathcal{Z}}) \le C_{\mathcal{M}} \epsilon^{-d^*}$. Assume the loss function $\ell$ is $L_\ell$-Lipschitz with respect to $d_{\mathcal{Z}}$. Then for any $\epsilon > 0$, with probability at least $1 - \delta$ over $S \sim \mathcal{D}^n$:
\begin{equation}
R(\pi_n) \le \hat{R}_S(\pi_n) + 2 L_{\ell} \epsilon + \sqrt{\frac{\mathrm{KL}_{\mathcal{P}_\epsilon}(\mathcal{Q}_n \,\|\, \mathcal{P}_0) + \ln \frac{2\sqrt{n}}{\delta}}{2n}},
\end{equation}
where $\mathcal{P}_\epsilon$ is an $\epsilon$-Voronoi partition with $K_\epsilon \le C_{\mathcal{M}} \epsilon^{-d^*}$.
\end{theorem}

\begin{proof}
Let $(\mathcal{Z}, d_{\mathcal{Z}})$ be a compact metric space with covering number $K_\epsilon = \mathcal{N}(\mathcal{Z}, \epsilon, d_{\mathcal{Z}})$. Let $\{c_1, \dots, c_{K_\epsilon}\} \subset \mathcal{Z}$ be a minimal $\epsilon$-net. Define the Voronoi partition $\mathcal{P}_\epsilon = \{A_1, \dots, A_{K_\epsilon}\}$ by:
\begin{equation}
A_k = \{z \in \mathcal{Z} : d_{\mathcal{Z}}(z, c_k) \le d_{\mathcal{Z}}(z, c_j) \; \forall j < k, \quad d_{\mathcal{Z}}(z, c_k) < d_{\mathcal{Z}}(z, c_j) \; \forall j > k\}.
\end{equation}
By construction, $\bigcup_{k=1}^{K_\epsilon} A_k = \mathcal{Z}$, $A_i \cap A_j = \emptyset$ for $i \neq j$, and $\sup_{z \in A_k} d_{\mathcal{Z}}(z, c_k) \le \epsilon$, which implies $\mathrm{diam}(A_k) \le 2\epsilon$.

Define the surrogate cell-centered loss $\bar{\ell}: \Theta \times \{1, \dots, K_\epsilon\} \to [0, 1]$ by $\bar{\ell}(\theta, k) = \ell(f_\theta(x_{c_k}), y_{c_k})$. For any $z = (x, y) \in A_k$, the $L_{\ell}$-Lipschitz assumption on $\ell$ yields:
\begin{equation}
|\ell(f_\theta(x), y) - \bar{\ell}(\theta, k)| \le L_{\ell} \cdot d_{\mathcal{Z}}(z, c_k) \le L_{\ell} \epsilon.
\end{equation}
Taking expectations over $z \sim \mathcal{D}$ and the sample $S$:
\begin{align}
|R(f_\theta) - \bar{R}(\theta)| &\le L_{\ell} \epsilon, \\
|\hat{R}_S(f_\theta) - \hat{\bar{R}}_S(\theta)| &\le L_{\ell} \epsilon.
\end{align}
Applying Theorem~\ref{thm:main-certificate} to the surrogate loss $\bar{\ell}$ over the finite partition $\mathcal{P}_\epsilon$:
\begin{equation}
\mathbb{E}_{\pi_n}[\bar{R}(f_\theta)] \le \mathbb{E}_{\pi_n}[\hat{\bar{R}}_S(f_\theta)] + \sqrt{\frac{\mathrm{KL}_{\mathcal{P}_\epsilon}(\mathcal{Q}_n \,\|\, \mathcal{P}_0) + \ln \frac{2\sqrt{n}}{\delta}}{2n}}.
\end{equation}
Applying the triangle inequality:
\begin{equation}
R(\pi_n) \le \bar{R}(\pi_n) + L_{\ell}\epsilon \le \hat{\bar{R}}_S(\pi_n) + L_{\ell}\epsilon + \sqrt{\dots} \le \hat{R}_S(\pi_n) + 2 L_{\ell}\epsilon + \sqrt{\frac{\mathrm{KL}_{\mathcal{P}_\epsilon}(\mathcal{Q}_n \,\|\, \mathcal{P}_0) + \ln \frac{2\sqrt{n}}{\delta}}{2n}}.
\end{equation}
Now, if $\mathcal{D}$ is supported on a $d^*$-dimensional Riemannian submanifold $\mathcal{M}$ with bounded curvature, by the Bishop--Gromov volume comparison theorem, the covering number satisfies $K_\epsilon \le C_{\mathcal{M}} \epsilon^{-d^*}$. 

By Stirling's approximation on Equation~\eqref{eq:exact-partition-formula}, $\mathrm{KL}_{\mathcal{P}_\epsilon} \le K_\epsilon \ln(1 + n/\alpha) \le C_{\mathcal{M}} \epsilon^{-d^*} \ln(n)$. Equating the discretization rate $2 L_{\ell} \epsilon$ with the estimation rate $\sqrt{C_{\mathcal{M}} \epsilon^{-d^*} \frac{\ln n}{2n}}$ yields:
\begin{equation}
\epsilon^{1 + \frac{d^*}{2}} \asymp \sqrt{\frac{\ln n}{n}} \implies \epsilon \asymp \left(\frac{\ln n}{n}\right)^{\frac{1}{d^* + 2}}.
\end{equation}
Substituting this optimal $\epsilon$ yields the manifold convergence rate:
\begin{equation}
R(\pi_n) - \hat{R}_S(\pi_n) \le \mathcal{O}\left( n^{-\frac{1}{d^* + 2}} \sqrt{\ln n} \right),
\end{equation}
which is completely independent of the ambient parameter dimension $D$.
\end{proof}

\subsection{Proof of Lemma~\ref{lem:optimal-partition} (Optimal Partition Scale and Minimax Trade-Off)}
\label{sec:proof-optimal-partition}

\begin{lemma}[Optimal Partition Scale and Minimax Trade-Off]
\label{lem:optimal-partition}
Under the assumptions of Theorem~\ref{thm:metric-entropy-bound} with uniform base measure $F_0(A_k) = 1/K_\epsilon$, the discretization-estimation trade-off objective is strictly convex in $\epsilon$, admitting a unique global minimizer:
\begin{equation}
\label{eq:optimal-epsilon-star}
\epsilon^* = \left( \frac{(d^*)^2 C_{\mathcal{M}} \ln(1 + n/\alpha)}{32 L_{\ell}^2 n} \right)^{\frac{1}{d^* + 2}},
\end{equation}
and corresponding optimal partition size:
\begin{equation}
\label{eq:optimal-K-star}
K^* = C_{\mathcal{M}}^{\frac{2}{d^*+2}} \left( \frac{32 L_{\ell}^2 n}{(d^*)^2 \ln(1 + n/\alpha)} \right)^{\frac{d^*}{d^*+2}} \asymp n^{\frac{d^*}{d^*+2}}.
\end{equation}
Evaluating at $(\epsilon^*, K^*)$ achieves the minimax optimal excess risk rate $R(\pi_n) - \hat{R}_S(\pi_n) = \mathcal{O}\left( n^{-\frac{1}{d^*+2}} \sqrt{\ln n} \right)$, completely independent of network parameter dimension $D$.
\end{lemma}

\begin{proof}
From Theorem~\ref{thm:metric-entropy-bound}, the excess risk certificate is bounded by:
\begin{equation}
R(\pi_n) - \hat{R}_S(\pi_n) \le 2 L_{\ell} \epsilon + \sqrt{\frac{\mathrm{KL}_{\mathcal{P}_\epsilon}(\mathcal{Q}_n \,\|\, \mathcal{P}_0) + \ln \frac{2\sqrt{n}}{\delta}}{2n}}.
\end{equation}
Under the uniform base measure $F_0(A_k) = 1/K_\epsilon$, the partition parameters are $a_k = \alpha / K_\epsilon$. Applying the Dirichlet partition divergence bound from Theorem~\ref{thm:closed-form-kl}:
\begin{equation}
\mathrm{KL}_{\mathcal{P}_\epsilon}(\mathcal{Q}_n \,\|\, \mathcal{P}_0) \le K_\epsilon \ln\left(1 + \frac{n}{\alpha}\right).
\end{equation}
Substituting the metric entropy upper bound $K_\epsilon \le C_{\mathcal{M}} \epsilon^{-d^*}$ on the $d^*$-dimensional compact submanifold $\mathcal{M}$, the bound is dominated by the objective functional:
\begin{equation}
\mathcal{E}_n(\epsilon) = 2 L_{\ell} \epsilon + \sqrt{\frac{C_{\mathcal{M}} \epsilon^{-d^*} \ln(1 + n/\alpha) + \ln \frac{2\sqrt{n}}{\delta}}{2n}}.
\end{equation}
By the sub-additivity inequality $\sqrt{u + v} \le \sqrt{u} + \sqrt{v}$ for non-negative $u, v$:
\begin{equation}
\mathcal{E}_n(\epsilon) \le g(\epsilon) + \sqrt{\frac{\ln \frac{2\sqrt{n}}{\delta}}{2n}},
\end{equation}
where $g: (0, \infty) \to (0, \infty)$ is defined by:
\begin{equation}
g(\epsilon) \coloneqq A \epsilon + B_n \epsilon^{-d^*/2}, \qquad \text{with} \quad A = 2 L_{\ell} > 0, \quad B_n = \sqrt{\frac{C_{\mathcal{M}} \ln(1 + n/\alpha)}{2n}} > 0.
\end{equation}
Differentiating $g(\epsilon)$ with respect to $\epsilon$:
\begin{equation}
g'(\epsilon) = A - \frac{d^*}{2} B_n \epsilon^{-\frac{d^*+2}{2}}, \qquad g''(\epsilon) = \frac{d^*(d^*+2)}{4} B_n \epsilon^{-\frac{d^*+4}{2}}.
\end{equation}
Because $d^* \ge 1$, $A > 0$, and $B_n > 0$, the second derivative satisfies $g''(\epsilon) > 0$ for all $\epsilon > 0$. Hence $g(\epsilon)$ is strictly convex on $(0, \infty)$ and possesses a unique global minimizer $\epsilon^*$. Setting $g'(\epsilon^*) = 0$ yields:
\begin{equation}
A = \frac{d^*}{2} B_n (\epsilon^*)^{-\frac{d^*+2}{2}} \implies (\epsilon^*)^{\frac{d^*+2}{2}} = \frac{d^* B_n}{2 A}.
\end{equation}
Squaring both sides:
\begin{equation}
(\epsilon^*)^{d^*+2} = \frac{(d^*)^2 B_n^2}{4 A^2} = \frac{(d^*)^2 \cdot \frac{C_{\mathcal{M}} \ln(1 + n/\alpha)}{2n}}{4 \cdot (2 L_{\ell})^2} = \frac{(d^*)^2 C_{\mathcal{M}} \ln(1 + n/\alpha)}{32 L_{\ell}^2 n}.
\end{equation}
Taking the $(d^*+2)$-th root establishes Equation~\eqref{eq:optimal-epsilon-star}:
\begin{equation}
\epsilon^* = \left( \frac{(d^*)^2 C_{\mathcal{M}} \ln(1 + n/\alpha)}{32 L_{\ell}^2 n} \right)^{\frac{1}{d^* + 2}}.
\end{equation}
Substituting $\epsilon^*$ into $K^* = C_{\mathcal{M}} (\epsilon^*)^{-d^*}$ immediately gives Equation~\eqref{eq:optimal-K-star}:
\begin{equation}
K^* = C_{\mathcal{M}} \left( \frac{32 L_{\ell}^2 n}{(d^*)^2 C_{\mathcal{M}} \ln(1 + n/\alpha)} \right)^{\frac{d^*}{d^*+2}} = C_{\mathcal{M}}^{\frac{2}{d^*+2}} \left( \frac{32 L_{\ell}^2 n}{(d^*)^2 \ln(1 + n/\alpha)} \right)^{\frac{d^*}{d^*+2}}.
\end{equation}
Evaluating the objective $g$ at $\epsilon^*$, noting that $B_n (\epsilon^*)^{-d^*/2} = \frac{2 A}{d^*} \epsilon^*$:
\begin{equation}
g(\epsilon^*) = A \epsilon^* + \frac{2 A}{d^*} \epsilon^* = \left( 1 + \frac{2}{d^*} \right) A \epsilon^* = \frac{d^* + 2}{d^*} \cdot 2 L_{\ell} \epsilon^*.
\end{equation}
Substituting the explicit formula for $\epsilon^*$ and re-introducing the confidence term:
\begin{equation}
\mathcal{E}_n(\epsilon^*) \le \frac{d^* + 2}{d^*} \cdot 2 L_{\ell} \left( \frac{(d^*)^2 C_{\mathcal{M}} \ln(1 + n/\alpha)}{32 L_{\ell}^2 n} \right)^{\frac{1}{d^*+2}} + \sqrt{\frac{\ln \frac{2\sqrt{n}}{\delta}}{2n}}.
\end{equation}
Since $\frac{\ln(1 + n/\alpha)}{n} = \mathcal{O}\left(\frac{\ln n}{n}\right)$, the dominant term scales as:
\begin{equation}
\mathcal{E}_n(\epsilon^*) = \mathcal{O}\left( n^{-\frac{1}{d^*+2}} (\ln n)^{\frac{1}{d^*+2}} \right) + \mathcal{O}\left( n^{-1/2} \sqrt{\ln n} \right) = \mathcal{O}\left( n^{-\frac{1}{d^*+2}} \sqrt{\ln n} \right),
\end{equation}
matching the minimax non-parametric convergence rate \citep{tsybakov2009introduction} on a $d^*$-dimensional manifold, and concluding the proof.
\end{proof}

\section{Stick-Breaking Representation \& Computational Pipeline}
\label{sec:stick-breaking-appendix}

\subsection{Sethuraman's Stick-Breaking Representation}
\label{sec:prior-stick-breaking}

By Sethuraman's construction \citep{sethuraman1994constructive}, a random probability measure $P \sim \mathcal{DP}(\alpha, F_0)$ is almost surely discrete and can be represented as an infinite sum of atomic Dirac measures:
\begin{equation}
P = \sum_{i=1}^\infty w_i \delta_{\tilde{z}_i}, \qquad \tilde{z}_i \stackrel{\mathrm{i.i.d.}}{\sim} F_0,
\end{equation}
where the stick-breaking weights $w_i$ are generated by successively breaking fractions from a unit-length stick:
\begin{equation}
w_i = V_i \prod_{j=1}^{i-1} (1 - V_j), \qquad V_i \stackrel{\mathrm{i.i.d.}}{\sim} \mathrm{Beta}(1, \alpha).
\end{equation}
The weights satisfy $\sum_{i=1}^\infty w_i = 1$ almost surely.

\subsection{Posterior Stick-Breaking via Latent Data Augmentation}
\label{sec:posterior-stick-breaking}

Conditioned on an observed empirical sample $S = (z_1, \dots, z_n)$, the posterior $\mathcal{Q}_n = \mathcal{DP}(\alpha + n, F_n)$ can be sampled via a two-component mixture:
\begin{equation}
P \mid S = \sum_{i=1}^n p_i \delta_{z_i} + \sum_{j=1}^\infty q_j \delta_{\tilde{z}_j},
\end{equation}
where $\{z_i\}_{i=1}^n$ are the observed training points and $\{\tilde{z}_j\}_{j=1}^\infty \stackrel{\mathrm{i.i.d.}}{\sim} F_0$ are synthetic base atoms. The composite weights are generated as follows:
\begin{enumerate}
    \item Draw total prior mass: $\gamma \sim \mathrm{Beta}(\alpha, n)$.
    \item Draw empirical weights: $(p_1, \dots, p_n) = (1 - \gamma) \cdot \mathbf{u}$, where $\mathbf{u} \sim \mathrm{Dir}(1, \dots, 1) \in \Delta_{n-1}$.
    \item Draw synthetic prior weights: $q_j = \gamma \cdot V_j \prod_{l=1}^{j-1} (1 - V_l)$, where $V_j \stackrel{\mathrm{i.i.d.}}{\sim} \mathrm{Beta}(1, \alpha)$.
\end{enumerate}

\subsection{Algorithmic Pipeline for DP-PAC-Bayes Bound Evaluation}
\label{sec:algorithm-appendix}

Algorithm~\ref{alg:dp-pac-bayes} provides the complete algorithmic procedure for computing the dimension-free PAC-Bayesian certificate.

\begin{algorithm2e}[t]
\caption{Computation of the Dimension-Free DP-PAC-Bayes Generalization Certificate}
\label{alg:dp-pac-bayes}
\KwIn{Training sample $S = \{(x_i, y_i)\}_{i=1}^n$, partition cells $\mathcal{P} = \{A_1, \dots, A_K\}$, concentration $\alpha > 0$, base measure $F_0$, confidence parameter $\delta \in (0, 1)$, loss function $\ell$.}
\KwOut{Rigorous Generalization Bound $R_{\mathrm{cert}}$.}

\BlankLine
\textbf{Step 1: Compute Prior Cell Masses}\\
\For{$k = 1$ \KwTo $K$}{
    $a_k \leftarrow \alpha \cdot F_0(A_k)$ \;
}

\BlankLine
\textbf{Step 2: Compute Empirical Cell Counts}\\
Initialize cell counts $n_k \leftarrow 0$ for all $k \in [K]$\;
\For{$i = 1$ \KwTo $n$}{
    Determine cell index $k$ such that $z_i = (x_i, y_i) \in A_k$\;
    $n_k \leftarrow n_k + 1$\;
}

\BlankLine
\textbf{Step 3: Analytical Relative Entropy Calculation}\\
Compute normalizer: $\Delta_0 \leftarrow \ln \Gamma(\alpha + n) - \ln \Gamma(\alpha)$\;
Compute log-gamma sum: $\Delta_\Gamma \leftarrow \sum_{k=1}^K \big( \ln \Gamma(a_k + n_k) - \ln \Gamma(a_k) \big)$\;
Compute digamma term: $\Delta_\psi \leftarrow \sum_{k=1}^K n_k \big( \psi(a_k + n_k) - \psi(\alpha + n) \big)$\;
$\mathrm{KL}_{\mathcal{P}}(\mathcal{Q}_n \,\|\, \mathcal{P}_0) \leftarrow \Delta_0 - \Delta_\Gamma + \Delta_\psi$\;

\BlankLine
\textbf{Step 4: Train Posterior Ensemble (Monte Carlo Sample Risk)}\\
\For{$j = 1$ \KwTo $J$}{
    Sample posterior measure $P^{(j)} \sim \mathcal{Q}_n$ via stick-breaking (Section~\ref{sec:posterior-stick-breaking})\;
    Compute optimal parameter $\theta^{(j)} \leftarrow \mathcal{A}(P^{(j)})$ via ERM / gradient descent\;
}
$\hat{R}_S(\pi_n) \leftarrow \frac{1}{J n} \sum_{j=1}^J \sum_{i=1}^n \ell(f_{\theta^{(j)}}(x_i), y_i)$\;

\BlankLine
\textbf{Step 5: Compute Final PAC-Bayes Certificate}\\
$R_{\mathrm{cert}} \leftarrow \hat{R}_S(\pi_n) + \sqrt{\frac{\mathrm{KL}_{\mathcal{P}}(\mathcal{Q}_n \,\|\, \mathcal{P}_0) + \ln(2\sqrt{n}/\delta)}{2n}}$\;
\Return $R_{\mathrm{cert}}$\;
\end{algorithm2e}

\section{Experimental Specifications \& Reproducibility}
\label{sec:hyperparams-appendix}

\subsection{Dataset Details and Manifold Specification}
The dataset is synthetically generated on an annular domain consisting of two concentric circular rings separated by an empty annular gap:
\begin{enumerate}
    \item \textbf{Inner Ring (Class $0$):} Radial coordinates $r_1 \sim \mathcal{U}(0.4, 1.1)$, angular coordinates $\theta_1 \sim \mathcal{U}(0, 2\pi)$, with isotropic Gaussian noise $\mathcal{N}(0, 0.12^2 I_2)$.
    \item \textbf{Outer Ring (Class $1$):} Radial coordinates $r_2 \sim \mathcal{U}(1.8, 2.6)$, angular coordinates $\theta_2 \sim \mathcal{U}(0, 2\pi)$, with isotropic Gaussian noise $\mathcal{N}(0, 0.12^2 I_2)$.
\end{enumerate}
The separation gap $r \in (1.1, 1.8)$ defines a non-linear decision boundary. A sample of $n = 20{,}000$ points ($10{,}000$ per class) is generated for training, and an independent sample of $n_{\mathrm{test}} = 5{,}000$ points is drawn for empirical ground-truth risk evaluation.

\subsection{Network Architecture Specifications}
The seven architectures evaluated in Section~\ref{sec:experiments} are fully connected feedforward networks with ReLU activations and a 2-dimensional softmax classification head:
\begin{enumerate}
    \item \textbf{Tiny MLP ($D = 42$):} Linear($2 \to 8$) $\to$ ReLU $\to$ Linear($8 \to 2$).
    \item \textbf{Shallow MLP ($D = 162$):} Linear($2 \to 32$) $\to$ ReLU $\to$ Linear($32 \to 2$).
    \item \textbf{3-Layer MLP ($D = 4,482$):} Linear($2 \to 64$) $\to$ ReLU $\to$ Linear($64 \to 64$) $\to$ ReLU $\to$ Linear($64 \to 2$).
    \item \textbf{4-Layer MLP ($D = 33,666$):} Linear($2 \to 128$) $\to 2 \times$ [Linear($128 \to 128$) $\to$ ReLU] $\to$ Linear($128 \to 2$).
    \item \textbf{Wide MLP ($D = 132,866$):} Linear($2 \to 256$) $\to 2 \times$ [Linear($256 \to 256$) $\to$ ReLU] $\to$ Linear($256 \to 2$).
    \item \textbf{Deep-5 MLP ($D = 790,530$):} Linear($2 \to 512$) $\to 3 \times$ [Linear($512 \to 512$) $\to$ ReLU] $\to$ Linear($512 \to 2$).
    \item \textbf{Overparam-6 ($D = 4,203,522$):} Linear($2 \to 1024$) $\to 4 \times$ [Linear($1024 \to 1024$) $\to$ ReLU] $\to$ Linear($1024 \to 2$).
\end{enumerate}

\subsection{Training and Optimization Hyperparameters}
All networks are trained using PyTorch with the Adam optimizer:
\begin{itemize}
    \item Base learning rate: $\eta = 5 \times 10^{-3}$.
    \item Batch size: $256$.
    \item Sample size: $n = 20{,}000$ training samples, $5{,}000$ test samples.
    \item Maximum epochs: $8$ (all models reach empirical interpolation $\hat{R}_S = 0.0\%$).
    \item Cross-entropy loss with 0-1 misclassification error evaluation.
    \item Weight decay: $\lambda = 10^{-4}$.
    \item PAC-Bayesian confidence parameter: $\delta = 0.05$ (95\% confidence).
    \item Partition cardinality: $K = 100$ cells ($10 \times 10$ uniform spatial grid on $[-3.5, 3.5]^2$).
\end{itemize}

\subsection{Synthetic Manifold Scaling Results}
\label{sec:synthetic-manifold-appendix}

To isolate the geometric behavior of data-space PAC-Bayes on a controlled manifold of known intrinsic dimension $d^* = 1$, we report the scaling certificates for the concentric annular manifold task ($n = 20{,}000$, $K = 100$ spatial cells) in Table~\ref{tab:synthetic-manifold-results}.

\begin{table}[h]
\centering
\caption{\textbf{Quantitative Comparison on the Concentric Annular Manifold ($n = 20{,}000$).} Ground truth test error, tolerable weight perturbation $\sigma_{\mathrm{post}}$, and resulting 95\% PAC-Bayesian generalization certificates.}
\label{tab:synthetic-manifold-results}
\vspace{0.4em}
\resizebox{\textwidth}{!}{%
\begin{tabular}{lcccccc}
\toprule
\textbf{Architecture} & \textbf{Parameters ($D$)} & \textbf{Train Error} & \textbf{Test Error} & \textbf{Noise $\sigma_{\mathrm{post}}$} & \textbf{Weight-Space Bound} & \textbf{DP-BNN Bound [Ours]} \\
\midrule
Tiny MLP         & $42$        & $0.0\%$ & $0.0\%$ & $1.0 \times 10^{-1}$ & $5.4\%$                        & \textbf{8.6\% (Certified)} \\
Shallow MLP      & $162$       & $0.0\%$ & $0.0\%$ & $1.0 \times 10^{-1}$ & $9.3\%$                        & \textbf{8.6\% (Certified)} \\
3-Layer MLP      & $4,482$     & $0.0\%$ & $0.1\%$ & $1.0 \times 10^{-1}$ & $45.1\%$                       & \textbf{8.6\% (Certified)} \\
4-Layer MLP      & $33,666$    & $0.0\%$ & $0.0\%$ & $1.0 \times 10^{-1}$ & $123.4\%$ \emph{(vacuous)}     & \textbf{8.6\% (Certified)} \\
Wide MLP         & $132,866$   & $0.0\%$ & $0.1\%$ & $1.0 \times 10^{-1}$ & $245.1\%$ \emph{(vacuous)}     & \textbf{8.6\% (Certified)} \\
Deep-5 MLP       & $790,530$   & $0.0\%$ & $0.0\%$ & $7.0 \times 10^{-2}$ & $654.7\%$ \emph{(vacuous)}     & \textbf{8.6\% (Certified)} \\
Overparam-6      & $4,203,522$ & $0.0\%$ & $0.0\%$ & $4.8 \times 10^{-2}$ & $1{,}630.9\%$ \emph{(vacuous)} & \textbf{8.6\% (Certified)} \\
\bottomrule
\end{tabular}%
}
\end{table}

\end{document}